\documentclass{article} 

\usepackage[utf8]{inputenc} 
\usepackage[T1]{fontenc}    
\usepackage{hyperref}       
\usepackage{url}            
\usepackage{booktabs}       
\usepackage{amsmath,amsfonts,amsthm}       
\usepackage{amssymb}
\usepackage{mathtools}      
\usepackage{nicefrac}       
\usepackage{microtype}      
\usepackage{subcaption}
\usepackage{mathabx}
\usepackage{bm}
\usepackage{dsfont}
\usepackage{multirow}
\usepackage[inline]{enumitem}
\usepackage{diagbox}
\usepackage{pifont}
\usepackage[capitalise]{cleveref}  
\usepackage{comment}
\usepackage{etoolbox}
\usepackage{xcolor}
\usepackage{graphicx}
\usepackage[ruled,vlined,linesnumbered]{algorithm2e}
\usepackage{algorithmic}
\usepackage[font=small]{caption}
\usepackage{tikz}
\usepackage{scalerel}

\usepackage[numbers,sort]{natbib}
\newlength{\defbaselineskip}
\newtheorem{theorem}{Theorem}

\newtheorem{lemma}{Lemma}

\allowdisplaybreaks

\DeclareMathOperator*{\argmin}{arg\,min}

\title{Non-Markovian Reinforcement Learning using Fractional Dynamics}

\usepackage{authblk}
\author[$\dagger$]{Gaurav Gupta}
\author[$\dagger$]{Chenzhong Yin}
\author[$\ddagger$]{Jyotirmoy V. Deshmukh}
\author[$\dagger$]{Paul Bogdan}
\affil[$\dagger$]{Ming Hsieh Department of Electrical and Computer 
Engineering,\\ University of Southern California, Los Angeles, CA, USA}
\affil[$\ddagger$]{Department of Computer Science, University of Southern California, Los Angeles, CA, USA}
\affil[ ]{{\texttt{\{ggaurav,chenzhoy,jdeshmuk,pbogdan\}@usc.edu}}}
\date{}

\begin{document}

\maketitle

\begin{abstract}
Reinforcement learning (RL) is a technique to learn the control policy for an agent that interacts with a stochastic environment. In any given state, the agent takes some action, and the environment determines the probability distribution over the next state as well as gives the agent some reward. Most RL algorithms typically assume that the environment satisfies Markov assumptions (i.e. the probability distribution over the next state depends only on the current state). In this paper, we propose a model-based RL technique for a system that has non-Markovian dynamics. Such environments are common in many real-world applications such as in human physiology, biological systems, material science, and population dynamics. Model-based RL (MBRL) techniques typically try to simultaneously learn a model of the environment from the data, as well as try to identify an optimal policy for the learned model. We propose a technique where the non-Markovianity of the system is modeled through a fractional dynamical system. We show that we can quantify the difference in the performance of an MBRL algorithm that uses bounded horizon model predictive control from the optimal policy. Finally, we demonstrate our proposed framework on a pharmacokinetic model of human blood glucose dynamics and show that our fractional models can capture distant correlations on real-world datasets.
\end{abstract}

\allowdisplaybreaks
\section{Introduction}
\label{sec:intro}

Reinforcement learning (RL) \cite{bertsekas2019reinforcement} is a technique to synthesize control policies for autonomous agents that interact with a stochastic environment.  The RL paradigm now contains a number of different kinds of algorithms, and has been successfully used across a diverse set of applications including autonomous vehicles, resource management in computer clusters \cite{mao2016resource}, traffic light control \cite{arel2010reinforcement},  web system configuration \cite{bu2009reinforcement}, and personalized recommendations \cite{zheng2018drn}.  In RL, we assume that in each state, the agent
performs some action and the environment picks a probability
distribution over the next state and assigns a reward (or negative cost). The reward is typically defined by the user with the help of a state-based (or state-action-based) reward function. The expected payoff that the agent may receive in any state can be defined in a number of different ways; in this paper, we assume that the payoff is an discounted sum of the local rewards (with some discount factor $\gamma \in [0,1]$) over some time horizon $H$. The purpose of RL is to find the stochastic policy (i.e. a distribution over actions conditioned on the current state), that optimizes the expected payoff for the agent. Most RL algorithms assume that the environment satisfies Markov assumptions, i.e. the probability distribution over the next state is dependent only on the current state (and not the history). In contrast, here, we investigate an RL procedure for a non-Markovian environment.

Broadly speaking, there are two classes of RL algorithms~\cite{chua2018deep}: model-based and model-free algorithms. Most classical RL algorithms are {\em model}-{\em based}; they assume that the environment is explicitly specified as a Markov Decision Process (MDP), and use dynamic programming to compute the expected payoff for each state of the MDP (called its value), as well as the optimal policy~\cite{sutton1999between,strehl2009reinforcement}. Classical RL algorithms have strong convergence guarantees stemming from the fact that the value of a state can be recursively expressed in terms of the value of the next state (called the Bellman equation), which allows us to define an operator to update the value (or the policy) for a given state across iterations. This operator (also known as the Bellman operator) can be shown to be a contraction mapping \cite{bertsekas2019reinforcement}. However, obtaining exact symbolic descriptions of models is often infeasible.  This led to the development of model-free reinforcement learning (MFRL) approaches that rely on sampling many model behaviors through simulations and eschew building a model altogether. MFRL algorithms can converge to an
optimal policy under the right set of assumptions; however, can suffer
from high sample complexity (i.e. the number of simulations required
to learn an optimal policy). This has led to investigation of a new class of model-based RL (MBRL) algorithms where the purpose is to  simultaneously learn the system model as well as the optimal policy \cite{mao2016resource}. Such algorithms are called {\em on policy}, as the policy learned during any iteration is used for improving the learned model as well as optimizing the policy further.  Most MBRL approaches use function approximators or Bayesian models to efficiently learn from scarce sample sets of system trajectories.  MBRL approaches tend to have lower sample complexity than MFRL as the learned model can accelerate the convergence by focusing on actions that are likely to be close to the optimal action. However, MBRL approaches can suffer severely from modeling errors \cite{todorov2012mujoco}, and may converge to less optimal solutions.

\begin{figure}
	\centering
     \includegraphics[width = 4.5in]{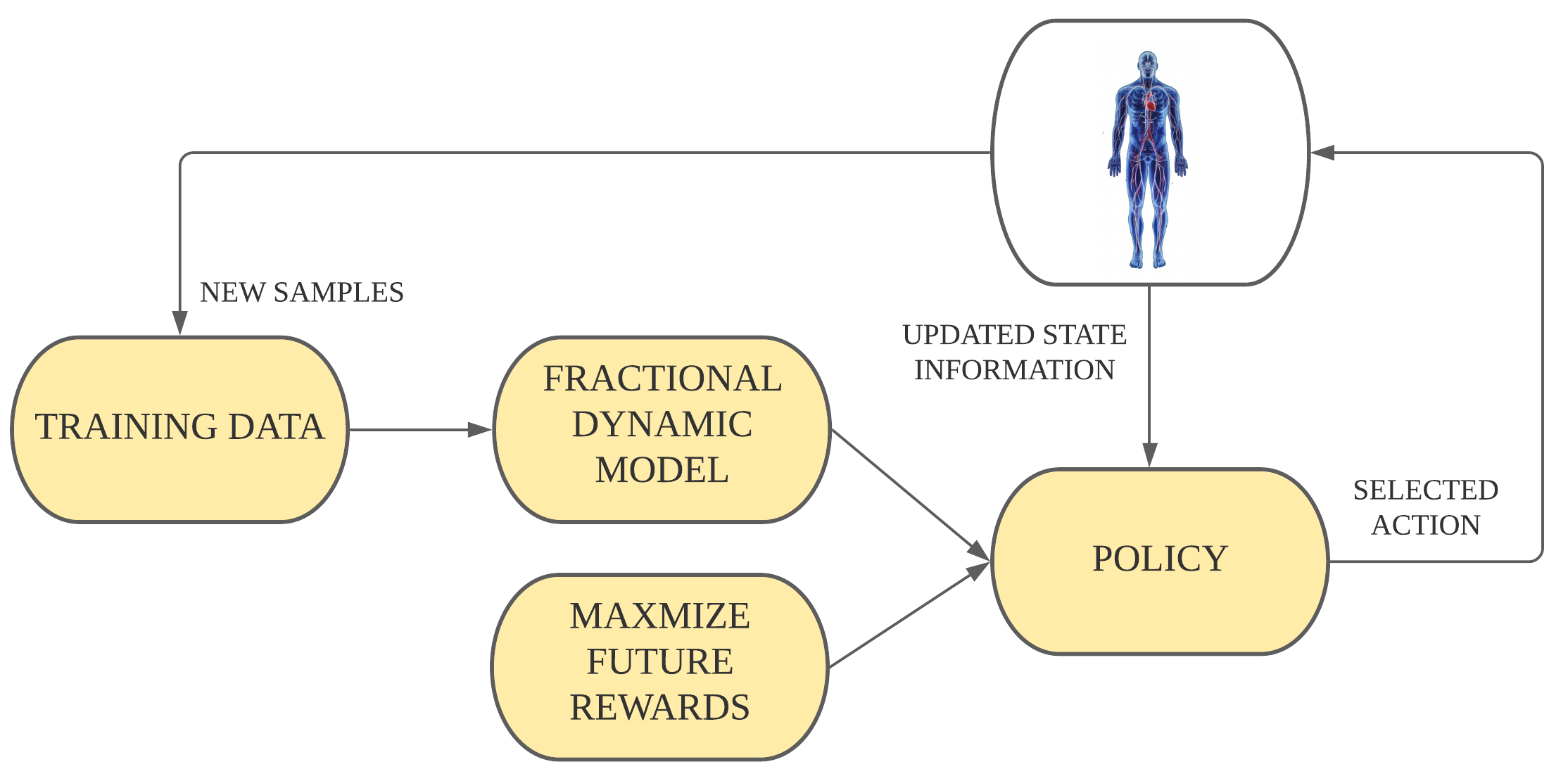}
	\caption{Non-Markovian Model Based Reinforcement Learning setup. The model based predictions are used to select actions, and then iteratively update the model dynamics.}
	\label{fig:fig_1}
\end{figure}

In both MFRL and MBRL algorithms, a fundamental assumption is that the environment satisfies Markovian properties, partly to avoid the complexity of dealing with the historical dependence in transitions. To overcome this challenge, we propose a non-Markovian MBRL framework that captures non-Markovian characteristics through a fractional dynamical systems formulation. Fractional dynamical systems can model non-Markovian processes characterized by a single fractal exponent and commonly arise in mathematical models of human physiological processes~\cite{West2010,xue}, biological systems, condensed matter and material sciences, and population dynamics \cite{gauravACC2018, yin2020discovering, gauravACC2019,gupta2018re}. Such systems can effectively model spatio-temporal properties of physiological signals such as blood oxygenation level dependent (BOLD), electromyogram (EMG), electrocardiogram (ECG), etc. \cite{gauravACC2018, baleanu2011fractional,magin2006fractional}. The advantage of using fractional dynamical models is that they can accurately represent long-range (historical) correlations (memory) through a minimum number of parameters (e.g., using a single fractal exponent to encode a long-range historical dependence rather than memorizing the trajectory itself or modeling it through a large set of autoregressive parameters). Though fractional models can be used to perform predictive control \cite{ghorbani2014}, problems such as learning these models effectively or obtaining optimal policies for such models in an RL setting have not been explored.

In this paper, we develop a novel non-Markovian MBRL technique
in which our algorithm alternates between incrementally learning the fractional exponent from data and learning the optimal policy on the updated model. We show that the optimal action in a given state can be efficiently computed by solving a quadratic program over a bounded horizon rollout from the state. The overview of our model-based reinforcement learning algorithm is shown in Fig.\,\ref{fig:fig_1}. In
this algorithm, we use on-policy simulations to gather additional RL data that is then used to update the model. Our model learning algorithm is based on minimizing the distance between the data's state-action distribution and the next state distribution induced by the controller. The fractional dynamic model is  then retrained using the cumulative dataset. The MBRL procedure is run for a finite number of user-specified iterations.

The rest of this paper is constructed as follows. We present our
problem statement in Section \ref{sec:probForm}. Section \ref{sec:rl}
contains our proposed non-Markovian MBRL algorithm. We demonstrate our
experimental results in Section \ref{sec:experi}. In the end,  we
conclude this paper with discussion and conclusion in Section
\ref{sec:concl}. 

\section{Problem Formulation}
\label{sec:probForm}
The reinforcement learning deals with the design of the controller (or policy) which minimizes the expected total cost. In the setting of a memoryless assumption, the Markov Decision Process (MDP) \cite{bellman1957} is used to model the system dynamics such that the future state depends only on the current state and action. For a state ${\bf s}_{t}\in \mathbb{R}^{n}$ and action ${\bf a}_{t}\in\mathbb{R}^{p}$, the future state evolve as ${\bf s}_{t+1}\sim P({\bf s}_{t+1}\vert {\bf s}_{t}, {\bf a}_{t})$, and a cost function $r_{t} = c({\bf s}_{t}, {\bf a}_{t})$. However, the Markov assumption does not work well with the long-range memory processes \cite{micciche2009modeling}. In this work, we take the non-Markovian setting, or History Dependent Process (HDP), and hence, the future state depends not only on the current action but also the history of states. The history at time $t$ is the set $\mathcal{H}_{t} = \{(s_{k})_{k\leq t}\}$, and for a trajectory $h\in\mathcal{H}_{t}$, we have $P({\bf s}_{t+1}\vert h,{\bf a}_{t})$, or alternatively, $P_{h}({\bf s}_{t+1}\vert s_{t},{\bf a}_{t})$, where the terminal state of the trajectory $h$ is written as $h(t) = s_{t}$. We consider a model-based approach for reinforcement learning in a finite-horizon setting. A non-Markovian policy $\pi(.\vert h)$ provides a distribution over actions given the history of states until time $t$ as $h\in\mathcal{H_{t}}$. For a given policy, the value function is defined as  $V_{h}^{\pi}= \mathbb{E}_{\pi(.\vert h)}\sum\nolimits_{t=0}^{T-1}c(s_{t}, a_{t})$, where the expectation is taken over state trajectories using policy $\pi$ and the HDP, and $T$ is the horizon under consideration. We formally define the non-Markovian MBRL problem in the Section\,\ref{ssec:nm_mbrl}.
\subsection{Fractional Dynamical Model}
\label{ssec:fracModel}
A linear discrete time fractional-order dynamical model is described as follows:
\begin{equation}
\Delta^{\alpha}{\bf s}[k+1] = {\bf A}{\bf s}[k] + {\bf B}{\bf a}[k],   
\label{eqn:fracLlinModel}
\end{equation}
\noindent where ${\bf s}\in\mathbb{R}^{n}$ is the state, ${\bf a} \in \mathbb{R}^{p}$ is the input action. The difference between a classic linear time-invariant (or Markovian) and the above model is the inclusion of fractional-order derivative whose expansion and discretization for any $i$th state $(1\leq i\leq n)$ can be written as
\begin{equation}
\Delta^{\alpha_{i}}s_{i}[k] = \sum\limits_{j=0}^{k}\psi(\alpha_{i},j) s_{i}[k-j],
\label{eqn:fracExpan}
\end{equation}
\noindent where $\alpha_{i}$ is the fractional order corresponding to the $i$th state dimension and $\psi(\alpha_{i},j) = \frac{\Gamma(j-\alpha_{i})}{\Gamma(-\alpha_{i})\Gamma(j+1)}$ with $\Gamma(.)$ denoting the gamma function. The system dynamics can also be written in the probabilistic manner as follows:
\begin{align}
&P_{{\bm \theta}}({\bf s}[k+1]\vert {\bf s}[0],\hdots,{\bf s}[k],{\bf a}[k]) = \mathcal{N}({\bm \mu}_{{\bm \theta}}, {\bm \Sigma}), \nonumber\\
&{\bm \mu}_{{\bm \theta}} = 
\begin{bmatrix*}[l]
\sum\nolimits_{j=1}^{k}\psi(\alpha_{i},j) s_{0}[k-j] + {\bf a}_{0}^{T}{\bf s}[k] + {\bf b}_{0}^{T}{\bf a}[k] + \mu_{0} \\
\sum\nolimits_{j=1}^{k}\psi(\alpha_{i},j) s_{1}[k-j] + {\bf a}_{1}^{T}{\bf s}[k] + {\bf b}_{1}^{T}{\bf a}[k] + \mu_{1} \\
\vdots\\
\sum\nolimits_{j=1}^{k}\psi(\alpha_{i},j) s_{n-1}[k-j] + {\bf a}_{n-1}^{T}{\bf s}[k] + {\bf b}_{n-1}^{T}{\bf a}[k] + \mu_{n-1}
\end{bmatrix*},
\label{eqn:hdp}
\end{align}
\noindent where ${\bm \theta} = \{\alpha, {\bf A}, {\bf B}, \mu, {\bm \Sigma}\}$, and ${\bf A} = [{\bf a}_{0},\hdots, {\bf a}_{n-1}]$, ${\bf B} = [{\bf
b}_{0},\hdots, {\bf b}_{n-1}]$. The fractional differencing operator in (\ref{eqn:hdp}) introduce the non-Markovianity by having long-range filtering operation on the state vectors.

\subsection{Non-Markovian Model Based Reinforcement Learning}
\label{ssec:nm_mbrl}
The actions in MBRL are preferred on the basis of predictions made by the undertaken model of the system dynamics. For many real-world systems, example blood glucose \cite{ghorbani2014,otoom2013real}, ECG activities \cite{xue}, the assumption of Markovian dynamics does not hold and hence the predictions are not accurate, leading to less rewarding actions selected for the system. As we note in the previous section \ref{ssec:fracModel} that non-Markovian dynamics can be effectively and compactly modeled as fractional dynamical system, we aim to use this system model for making predictions. The non-Markovian MBRL problem is formally defined as follows.

\noindent\textbf{Problem Statement:} Given non-Markovian state transitions, and actions dataset in the time horizon $k\in [0, T-1]$ as $\mathcal{D} = \big\{({\bf s}[0],\hdots, {\bf s}[k], {\bf a}[k]),{\bf s}[k+1]\big\}$. Let $P_{{\bm \theta}}({\bf s}[k+1]\vert {\bf s}[0],\hdots,{\bf s}[k],{\bf a}[k])$ be the non-Markovian system dynamics parameterized by the model parameters ${\bm \theta}$. Estimate the optimal policy which minimizes the expected future discounted cost
\begin{equation}
\begin{aligned}
\pi^{\ast} =\argmin\limits_{\pi}\,\mathbb{E}\sum\limits_{k=0}^{T-1}\gamma^{k} c({\bf s}[k], {\bf a}[k]),
\end{aligned}
\label{eqn:rew1}
\end{equation}
\noindent where $\gamma$ is the discount factor satisfying $\gamma\in[0, 1]$, and T is the horizon under consideration.

\section{Non-Markovian Reinforcement Learning}
\label{sec:rl}
The MBRL comprises of two key steps, namely (\emph{i}) the estimation of the model dynamics from the given data $\mathcal{D}$, and (\emph{ii}) the design of a policy for optimal action selection which minimizes the total expected cost using estimated dynamics. We discuss the solution to the non-Markovian MBRL as follows.

\subsection{Non-Markovian Model Predictive Control}
\label{ssec:mpc}
The Model Predictive Control (MPC) aims at estimating the closed-loop policy by optimizing the future discounted cost under a limited-horizon $H$ using some approximation of the environment dynamics and the cost. In this work, we are concerned with HDP using non-Markovian state dynamics. In MPC, the policy could be a deterministic action, or a distribution over actions, and we sample the action at each time-step in the latter. The MPC problem to estimate the policy at time-step $k$ for a given $h\in\mathcal{H}_{k}$ can be formally defined as
\begin{align}
\begin{aligned}
&\min\limits_{\pi(.\vert h)}&&\sum\nolimits_{l=k}^{k+H-1}\gamma^{l-k} \hat{c}(s[l], a[l]) \\
&\text{subject to}& \\
&&& s[l+1] = f(h, a[l], e[l]), \forall l\geq k\\
\end{aligned}
\label{eqn:mpc_general}
\end{align}
The approximation of the environment dynamics $f$ could be non-linear in general, and $e[l]$ is the system perturbation noise following some distribution $e\sim g_{e}$. The presence of $e$ provides randomness in the action sampling through policy, and the sampled action at each step is $a[k]$. The performance of the non-Markovian MPC based policy is bounded within the optimal policy using the following result.
\begin{theorem}
Given an approximate HDP with $\vert\vert\hat{P}_{h^{\prime}}(s^{\prime}\vert s, a) - P_{h}(s^{\prime}\vert s, a)\vert\vert_{1}\leq \mathcal{O}(t^{q})$, $\forall h, h^{\prime}\in \mathcal{H}_{t}$ with $h(t)=h^{\prime}(t)=s$, and $\vert\vert c(s,a)-\hat{c}(s,a)\vert\vert_{\infty}\leq\varepsilon$. The performance of the non-Markovian MPC based policy $\hat{\pi}$ is related to the optimal policy $\pi^{*}$ as
\begin{align}
\vert\vert V_{h_{0}}^{\hat{\pi}} - V_{h_{0}}^{\pi^{*}}\vert\vert_{\infty}&\leq 2\frac{1-\gamma^{H}}{1-\gamma}\left(\frac{c_{max}-c_{min}}{2}\right)H\mathcal{O}(T^{q})+ 2\varepsilon\frac{1-\gamma^{H}}{1-\gamma}\frac{1-\gamma^{T}}{1-\gamma},
\end{align}
where, $h_{0}\in\mathcal{H}_{0}$ is the initial history given to the system.
\end{theorem}
The assumption of model approximation is critical here, and the error increases if the exponent $q$ increases. For the MDP setting, the approximation is taken as $q=0$. However, for a HDP with the history of length $t$, we scale the approximation gap with $t$. The MPC horizon also plays a role in the error bound, and the error increases for larger $H$.

The non-Markovian MPC could be computationally prohibitive (expensive) in the general setting. Consequently, we now discuss the fractional dynamical MPC approach which is non-Markovian but computationally tractable.

\subsection{Fractional Model Predictive Control}
\label{ssec:frac_mpc}
The linear discrete fractional dynamical model as discussed in \eqref{eqn:fracLlinModel} is used as an approximation to the non-Markovian environment dynamics.
Formally, for our purpose, the fractional MPC problem using \eqref{eqn:mpc_general} is defined as  
\begin{align}
\begin{aligned}
&&\min\limits_{{\bf a}[k]}\sum\nolimits_{l=k}^{k+H-1}&\gamma^{l-k} \hat{c}({\bf s}[l], {\bf a}[l]) \\
&\text{s.t.}& \\
&&\Delta^{\alpha}\bar{\bf s}[l+1] &= {\bf A}\bar{\bf s}[l] + {\bf B}{\bf a}[l] + e[l],\\
&&\bar{\bf s}[k^{\prime}] &= {\bf s}[k],\forall k^{\prime}\leq k, \\
&&{\bf s}_{min} &\leq \bar{\bf s}[l]\leq {\bf s}_{max},\forall l,
\end{aligned}
\label{eqn:mpc_frac}
\end{align}
\noindent where ${\bf s}_{min}, {\bf s}_{max}$ are feasibility bounds on the problem according to the application, and the model noise $e\sim\mathcal{N}(0, \Sigma)$. Note that \eqref{eqn:mpc_frac} provides a policy using fractional MPC. The action $a[k]$ is sampled from this policy by first sampling $e\sim\mathcal{N}(0, \Sigma)$, and then solving \eqref{eqn:mpc_frac}. The non-Markovian fractional dynamics would introduce the computation complexities in optimally solving the problem in \eqref{eqn:mpc_frac}. However, since the constraints in \eqref{eqn:mpc_frac} are linear, for cost approximations $\hat{c}$ that are quadratic, a quadratic programming (QP) solution can be developed to solve the fractional MPC efficiently. 
We refer the reader to Appendix\,\ref{appn:fracMPCQprog} for the QP version of the fractional MPC. 
Further, a convex formulation of the costs $\hat{c}$ also enables efficient solution of the fractional MPC using convex programming solvers, for example, CPLEX and Gurobi \cite{kronqvist2019review, hutter2010automated}.

Next, we discuss the methodologies required to make an approximation of the non-Markovian environment using fractional dynamics. 
\subsection{Model Estimation}
\label{ssec:modelEst}
The fractional dynamical model as described in the  Section\,\ref{ssec:fracModel} is estimated using the approach proposed in \cite{gauravACC2018} by replacing the unknown inputs with known actions at any time-step. For the sake of completeness, we present estimation algorithm as Algorithm\,\ref{alg:frac_model}. We note that in \cite{gauravACC2018} the input data is obtained only once, and hence in this work appropriate modification in Algorithm\,\ref{alg:frac_model} is performed to work with recursively updated dataset as we see in Section\,\ref{ssec:mbrl}.
\begin{algorithm}
    \textbf{Input:} $\mathcal{D} = \big\{({\bf s}[0],\hdots, {\bf s}[k], {\bf a}[k]),{\bf s}[k+1]\big\}$ in the time-horizon $k\in [0, T - 1]$\\
    \textbf{Output:} ${\bm \theta} = \{\alpha, {\bf A}, {\bf B}, \mu, {\bf \Sigma}\}$
    \begin{algorithmic}[1]
        \STATE Estimate $\alpha$ using wavelets fitting for each state dimension 
        \FOR{$i = 1, 2, \hdots, n$}
        \STATE Compute $z_{i}[k] = \Delta^{\alpha_{i}}s[k+1]$ using $\alpha_{i}$ \hfill$\triangleright$ Eq.\eqref{eqn:fracExpan}
        \STATE Aggregate $z_{i}[k], s[k], a[k]$ as $Z_{i}, S, U$
        \STATE $[a_{i}^{T}, b_{i}^{T}, \mu] = \arg\min\limits_{a, b, \mu}\vert\vert Z_{i} - Sa - Ub - \mu\vert\vert_{2}^{2}$ with ${\bm \Sigma}$ as squared error
        \ENDFOR
    \end{algorithmic}
    \caption{Fractional\_Dynamics\_Estimation}
    \label{alg:frac_model}
\end{algorithm}

The Markovian model assume memoryless property and hence lacks long-range correlations for further accurate modeling. The existence of long-range correlations can be estimated by computing the Hurst exponent $\bar{H}$. For long-range correlations, the $\bar{H}$ lies in the range of $(0.5, 1]$. The fractional coefficient $\alpha$ in our model is related with $\bar{H}$ as $\alpha = \bar{H} - 0.5$. The Hurst exponent can be estimated from the slope of log-log variations of the variance of wavelets coefficients vs scale as noted in \cite{flandrin}. In the experiments Section\,\ref{ssec:real_world}, we show log-log plot to observe the presence of long-range correlations in the real-world data.

\subsection{Model Based Reinforcement Learning}
\label{ssec:mbrl}
The non-Markovian MPC exploiting the fractional dynamical model formulation in Section\,\ref{ssec:frac_mpc} utilizes a dataset of the form $\mathcal{D} = \big\{({\bf s}[0],\hdots, {\bf s}[k], {\bf a}[k]),{\bf s}[k+1]\big\}$ in the time-horizon $k\in [0, T-1]$. We note that the performance of such MPC can be further improved by using reinforcement learning. The selected actions by the MPC ${\bf a}[k]$ can be used to gather new transitions ${\bf s}[k+1]\vert {\bf s}[0],\hdots,{\bf s}[k], {\bf a}[k]$, or acquiring data using on-policy. The aggregated data is now used to re-estimate the model dynamics, and then perform MPC. Specifically, the MBRL proceeds as follows. Using the seed dataset, a parameterized fractional model dynamics is estimated as $P_{{\bm \theta}}({\bf s}[k+1]\vert {\bf s}[0],\hdots,{\bf s}[k],{\bf a}[k])$. The model dynamics is used to minimize the discounted future cost as MPC in equation (\ref{eqn:mpc_frac}). The selected action along with the history of states ${\bf s}[0],\hdots,{\bf s}[k]$ is used to gather the next transition using on-policy as  ${\bf s}[k+1]\vert {\bf s}[0],\hdots,{\bf s}[k], {\bf a}[k]$. The seed dataset is updated with the gathered on-policy data $\mathcal{D}_{RL}$ to get aggregated dataset. The fractional dynamics are updated using the new dataset, and the aforementioned steps are repeated for a given number of iterations. The above steps are summarized as Algorithm\,\ref{alg:frl}. The Algorithm\,\ref{alg:frl} utilizes  Algorithm\,\ref{alg:frac_model} iteratively for the fractional model estimation. We now proceed to Section\,\ref{sec:experi} for numerical demonstration of the proposed schemes.
\begin{algorithm}
    \textbf{Input:} Seed dataset $\mathcal{D}_{s} = \big\{({\bf s}[0],\hdots, {\bf s}[k], {\bf a}[k]),{\bf s}[k+1]\big\}$ in the time-horizon $k\in [0, T - 1]$\\
    \textbf{Output:} ${\bm \theta}$\\
    \textbf{Initialize:} $\mathcal{D}_{RL} \leftarrow \phi$
    \begin{algorithmic}[1]
        \FOR{$iter = 1, 2, \hdots, iter\_max$}
        \STATE $\theta \leftarrow \text{Fractional}\_\text{Dynamics}\_\text{Estimation}(\mathcal{D}_{s}\cup\mathcal{D}_{RL})$
        \STATE Set initial state $\bar{\bf s}[0] \leftarrow {\bf s}[0]$
        \FOR{$k = 0, 1, \hdots, T-1$}
        \STATE Sample action $a[k]$ from the fractional MPC based policy using $\bar{{\bf s}}[l],\forall l\leq k$ \hfill$\triangleright$ Eq.\eqref{eqn:mpc_frac}
        \STATE Get $\bar{\bf s}[k+1]$ by executing ${\bf a}[k]$
        \STATE $\mathcal{D}_{RL}\leftarrow\mathcal{D}_{RL}\cup\big\{(\bar{\bf s}[0],\hdots, \bar{\bf s}[k], {\bf a}[k]), \bar{\bf s}[k+1] \big\}$
        \ENDFOR
        \ENDFOR
    \end{algorithmic}
    \caption{Fractional\_Reinforcement\_Learning}
    \label{alg:frl}
\end{algorithm}
\section{Experiments}
\label{sec:experi}

We show example of the fractional MBRL on the blood glucose (BG) control. The motive of blood glucose control is to make the BG in the range of $70-180 mg/dL$. The BG control is crucial in the treatment of T1 diabetes patients which have inability to produce the required insulin amounts. The low levels of glucose in the blood plasma is termed as hypoglycemia, while the high levels is termed as hyperglycemia. For the application of reinforcement learning, the cost function is taken as risk associated with different levels of BG in the system. In \cite{pmid19469677} a quantified version of risk is proposed as function of BG levels which is written as follows.
\begin{align}
f(b) &= 1.509 \times (\log(b)^{1.084} - 5.381), \nonumber \\
R(b) &= 10\times(f(b))^2.
\label{eqn:risk_def}
\end{align}

Next, the cost for the transition instance $s[k+1]\vert s[0],\hdots, s[k], a[k])$ is written as
\begin{equation}
    \hat{c}(s[k], a[k]) = R(s[k+1]) - R(s[k]),
\end{equation}
\noindent where the state $s[k]\in\mathbb{R}$ represents the BG level at time instant $k$, and $a[k]$ represents the insulin dose and $R(.)$ is from \eqref{eqn:risk_def}. In rest of the section, we experiment with simulated and real-world dataset, respectively.
\subsection{UVa T1DM Simulator}

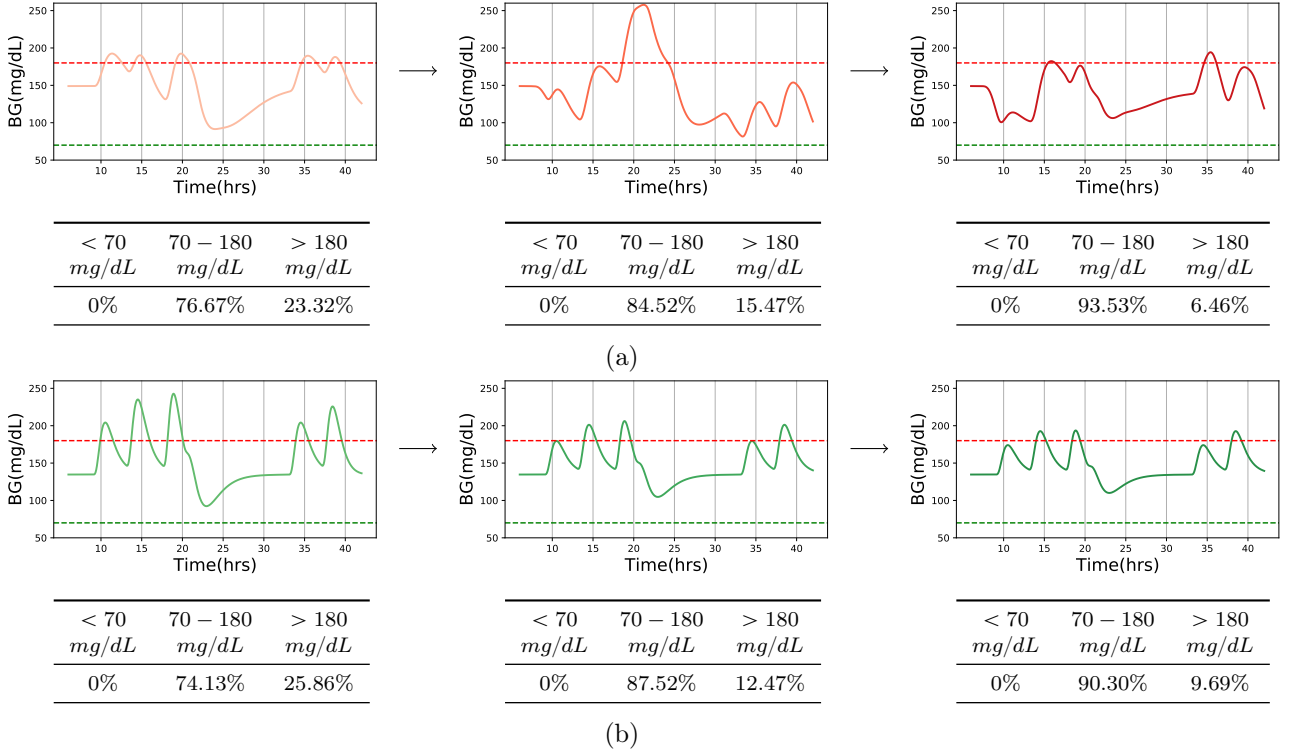
\begin{figure}[!t]
\centering
\begin{minipage}{\textwidth}
\centering
\begin{tikzpicture}
\node[anchor=north west,inner sep=0] at (0,0) {\includegraphics*[width = 2in, ]{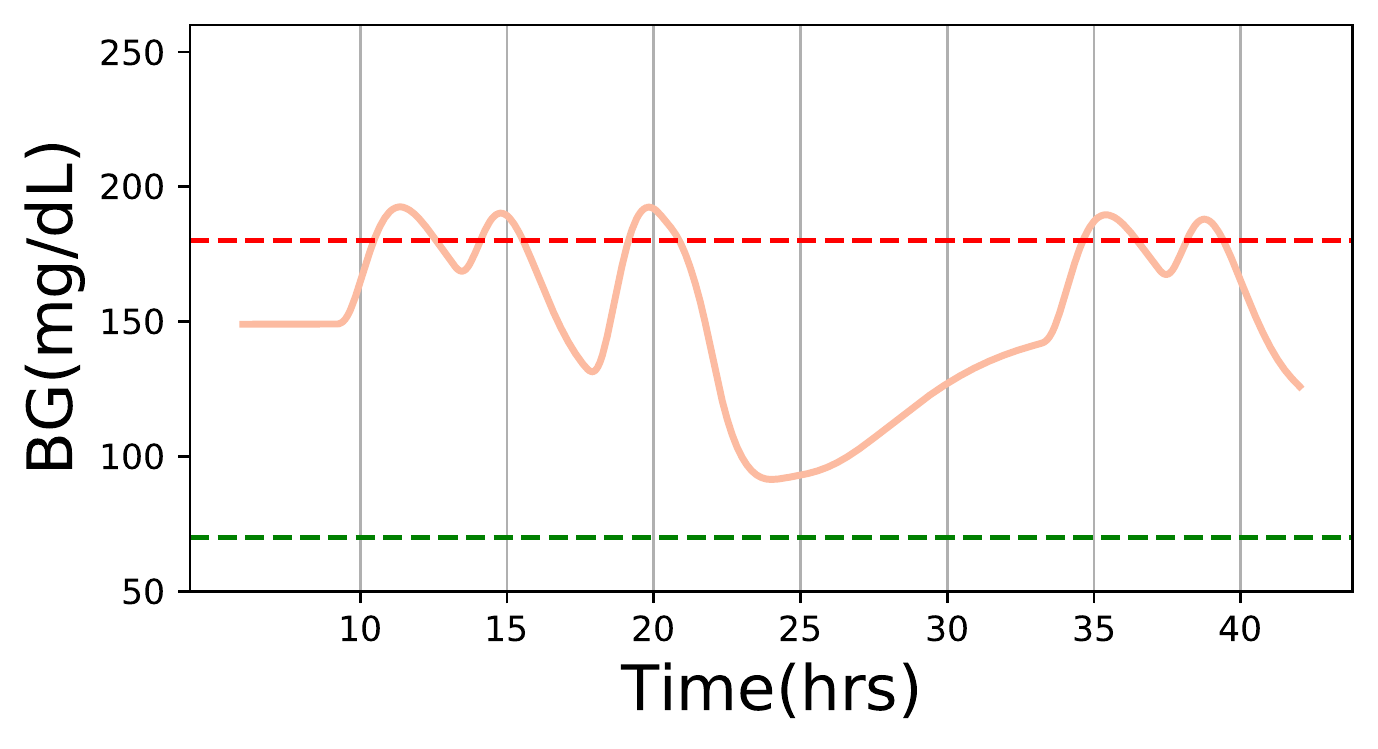}};
\draw [->] (5.3,-1) -- (5.8,-1);
\draw [->] (5.3+6,-1) -- (5.8+6,-1);
\node[anchor=north west,inner sep=0] at (6,0) {\includegraphics*[width = 2in, ]{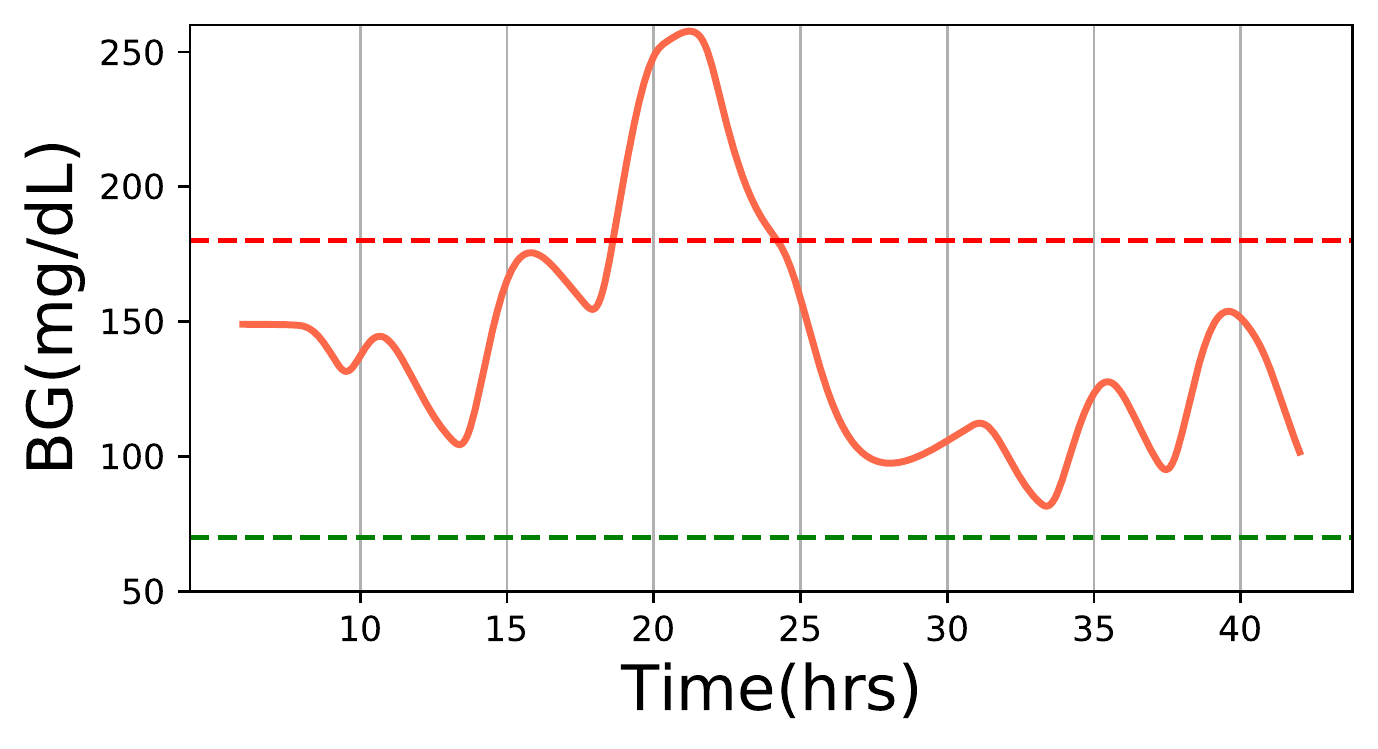}};
\node[anchor=north west,inner sep=0] at (12,0) {\includegraphics*[width = 2in, ]{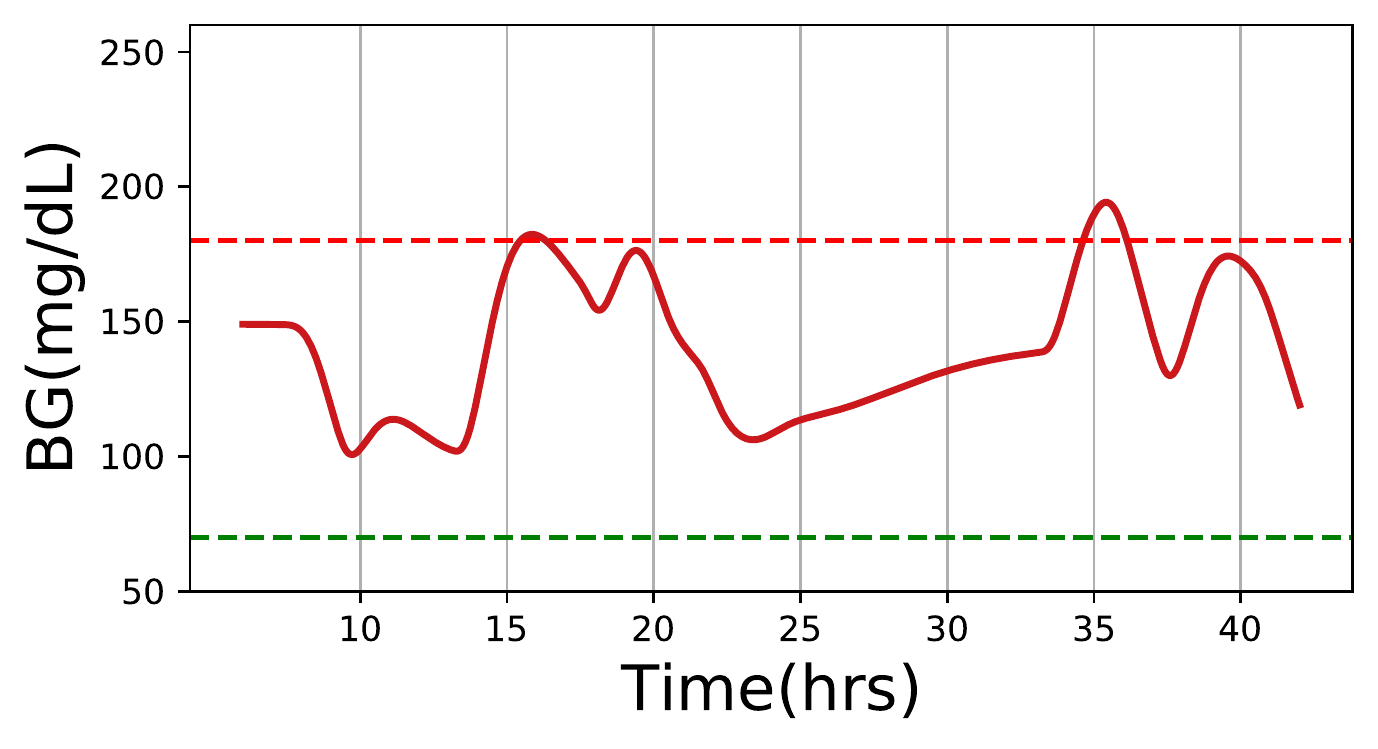}};
\node[anchor=north west,inner sep=0] at (0.7,-3) {\footnotesize
    \begin{tabular}{ccc}
    \toprule
    $< 70$  & $70-180$ & $> 180$ \\
    $mg/dL$ & $mg/dL$ & $mg/dL$\\
    \midrule
     $0 \%$& $76.67\%$ & $23.32\%$ \\
     \midrule
    \end{tabular}
};
\node[anchor=north west,inner sep=0] at (6.7,-3) {\footnotesize
    \begin{tabular}{ccc}
    \toprule
    $< 70$  & $70-180$ & $> 180$ \\
    $mg/dL$ & $mg/dL$ & $mg/dL$\\
    \midrule
     $0 \%$& $84.52\%$ & $15.47\%$ \\
     \midrule
    \end{tabular}
};
\node[anchor=north west,inner sep=0] at (12.7,-3) {\footnotesize
    \begin{tabular}{ccc}
    \toprule
    $< 70$  & $70-180$ & $> 180$ \\
    $mg/dL$ & $mg/dL$ & $mg/dL$\\
    \midrule
     $0 \%$& $93.53\%$ & $6.46\%$ \\
     \midrule
    \end{tabular}
};
\end{tikzpicture}\\
(a)
\end{minipage}
\\
\begin{minipage}{\textwidth}
\centering
\begin{tikzpicture}
\node[anchor=north west,inner sep=0] at (0,0) {\includegraphics[width = 2in, ]{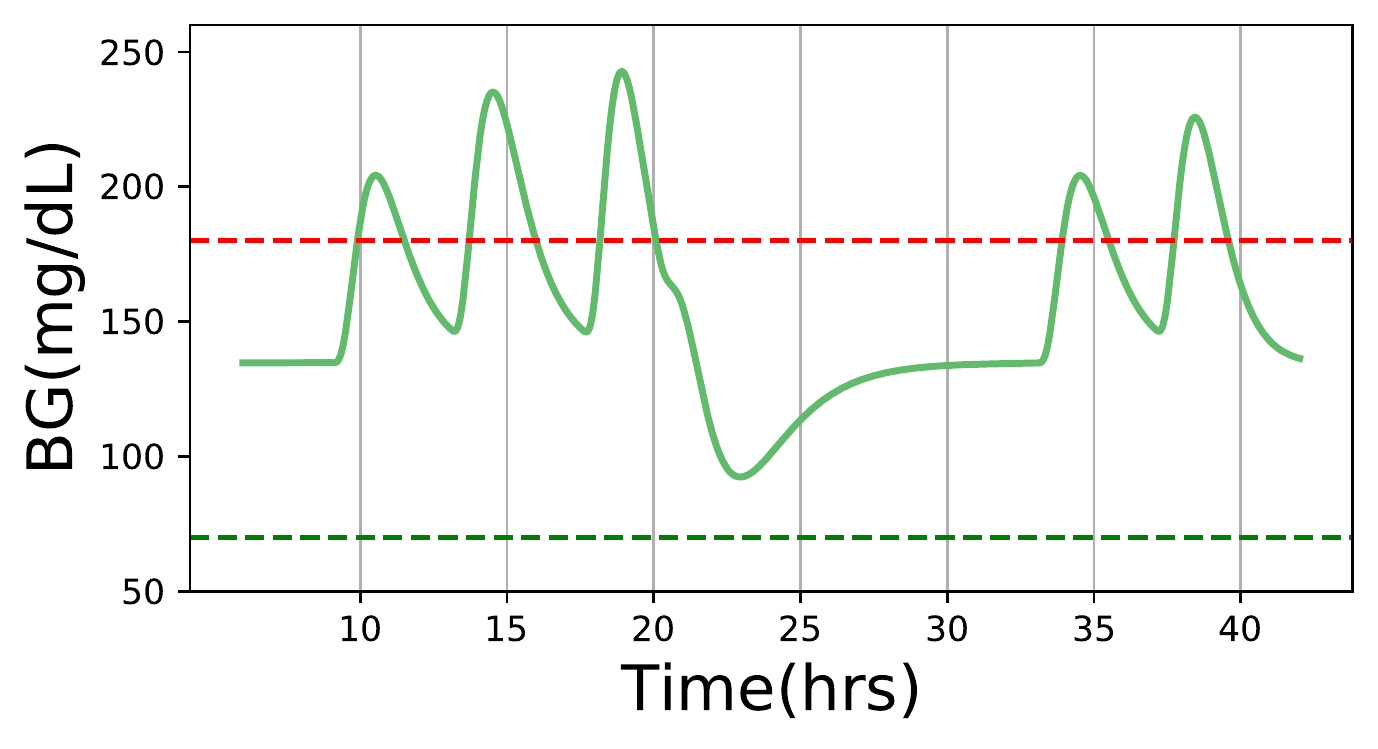}};
\draw [->] (5.3,-1) -- (5.8,-1);
\draw [->] (5.3+6,-1) -- (5.8+6,-1);
\node[anchor=north west,inner sep=0] at (6,0) {\includegraphics[width = 2in, ]{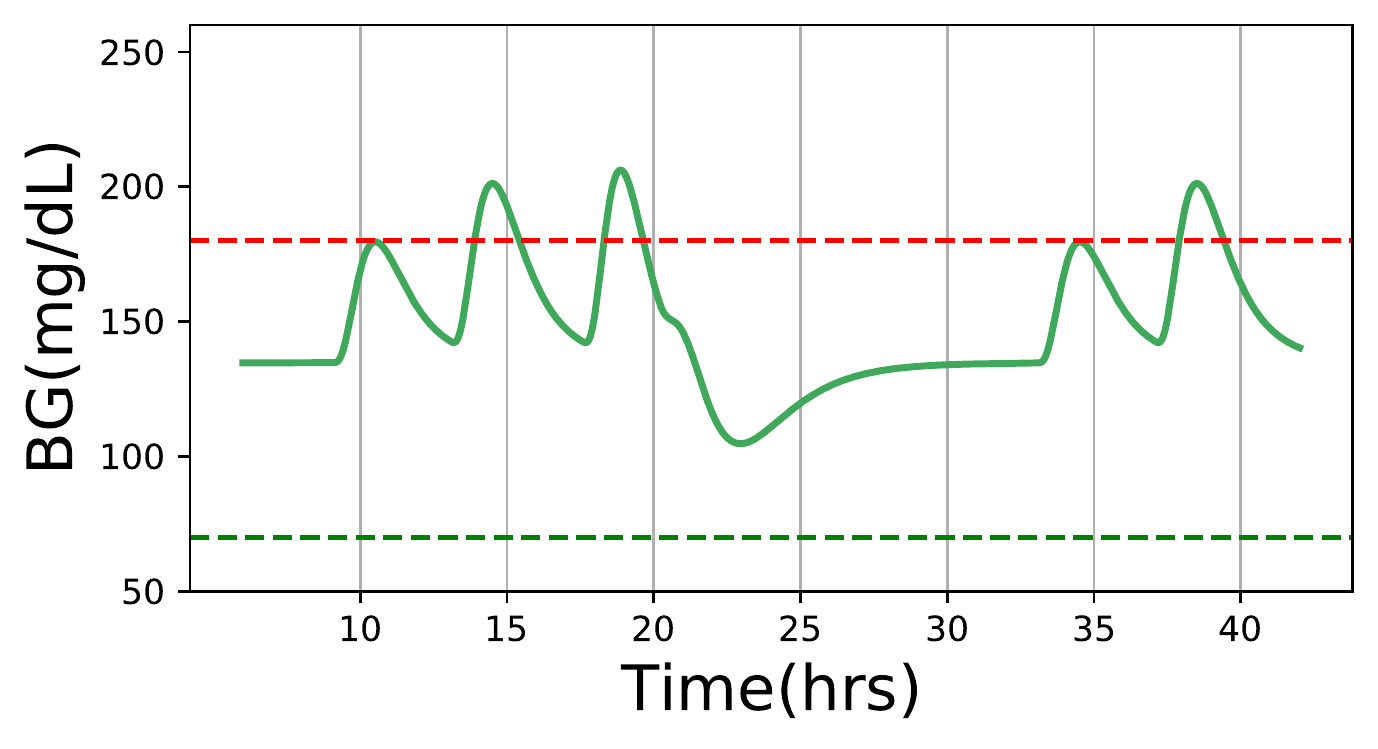}};
\node[anchor=north west,inner sep=0] at (12,0) {\includegraphics[width = 2in, ]{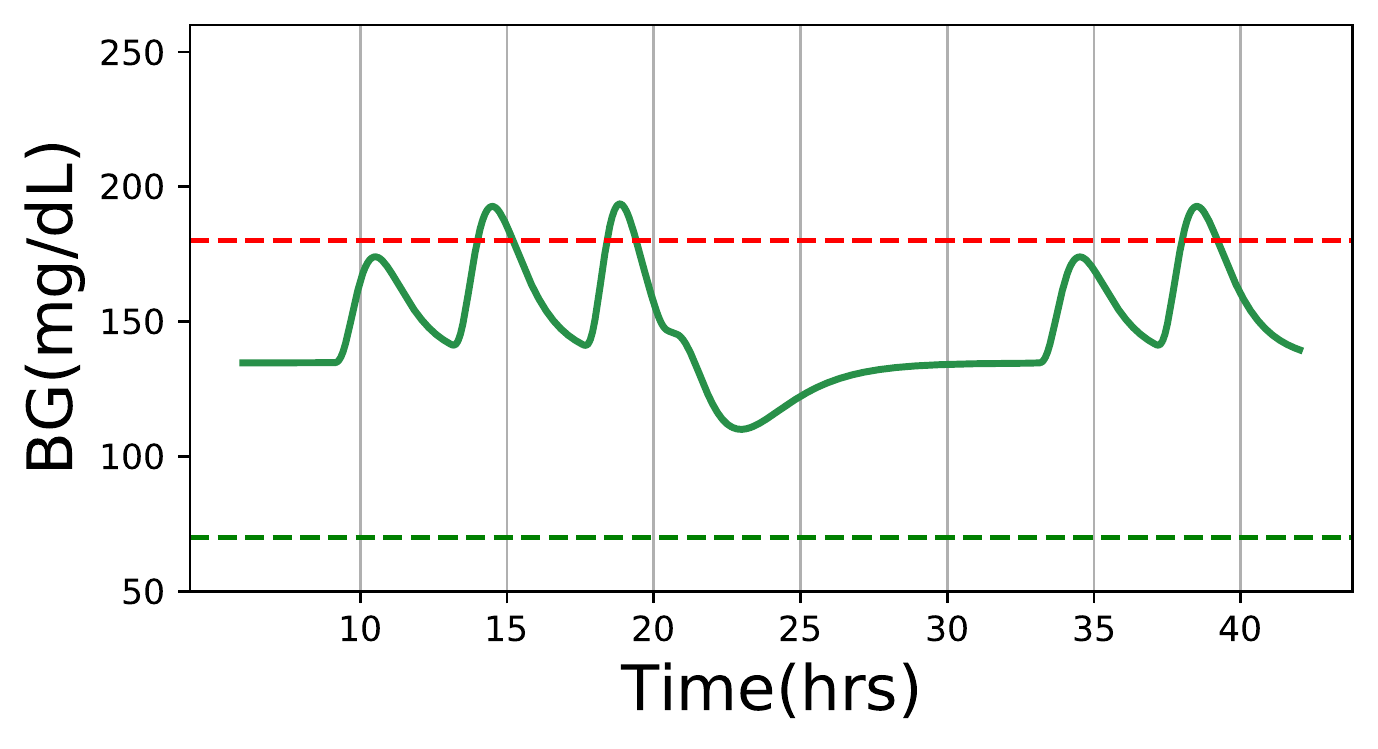}};
\node[anchor=north west,inner sep=0] at (0.7,-3) {\footnotesize
    \begin{tabular}{ccc}
    \toprule
    $< 70$  & $70-180$ & $> 180$ \\
    $mg/dL$ & $mg/dL$ & $mg/dL$\\
    \midrule
     $0 \%$& $74.13\%$ & $25.86\%$ \\
     \midrule
    \end{tabular}
};
\node[anchor=north west,inner sep=0] at (6.7,-3) {\footnotesize
    \begin{tabular}{ccc}
    \toprule
    $< 70$  & $70-180$ & $> 180$ \\
    $mg/dL$ & $mg/dL$ & $mg/dL$\\
    \midrule
     $0 \%$& $87.52\%$ & $12.47\%$ \\
     \midrule
    \end{tabular}
};
\node[anchor=north west,inner sep=0] at (12.7,-3) {\footnotesize
    \begin{tabular}{ccc}
    \toprule
    $< 70$  & $70-180$ & $> 180$ \\
    $mg/dL$ & $mg/dL$ & $mg/dL$\\
    \midrule
     $0 \%$& $90.30\%$ & $9.69\%$ \\
     \midrule
    \end{tabular}
};
\end{tikzpicture}\\
(b)
\end{minipage}
\caption{Blood Glucose (BG) level with time, by implementation of fractional Reinforcement Learning Scheme as Controller, of two Adults in (a) and (b). For each subject, the BG level trajectories are shown from left-to-right in the increasing number of RL iterations with leftmost, middle, and rightmost are outputs at $5$, $10$, and $15$ iterations. As RL iterations increase the MBRL scheme learns better policy and the BG level stays more in the desired level of $70-180 mg/dL$. The percentage of time spend in different BG level zone is shown in tables below each plot.}
\label{fig:bg_sim}
\end{figure}

The UVa/Padova T1DM \cite{pmid19444330} is a FDA approved T1 Diabetes simulator which supports multiple virtual subjects. An open-source implementation of the simulator \cite{xieSimulator} is used in this work. We take similar simulation setup as in \cite{wang2015acc}. Each subject is simulated for a total of $36$ hours starting from 6 a.m. in the morning. The meal timings/quantity are fixed as $50g$ CHO at 9 a.m., $70g$ at 1 p.m, $90g$ at 5:30 p.m, and $25g$ at 8 p.m. On day 2, $50g$ at 9 a.m., and $70g$ at 1 p.m. The continuous glucose monitor (CGM) sensor measures the BG at every 5 mins.

For applying Algorithm\,\ref{alg:frl}, we set the horizon length $H$ in MPC be $100$ samples, discount factor $\gamma = 0.99$. The $s_{min}, s_{max}$ in MPC problem \eqref{eqn:mpc_frac} are set as $70, 180$ respectively. The maximum number of RL iterations $iter\_max$ are set as $30$. We show the BG output of the simulator using Algorithm\,\ref{alg:frl} as controller in Fig.\,\ref{fig:bg_sim}. We observe that the fraction of time BG stays in the desired zone $70-180 mg/dL$ increase with increasing the learning iterations in the Algorithm\,\ref{alg:frl}. The data gathered using on-policy helps the model making better prediction, and with as few iterations as $15$ we have more than $90\%$ of time BG stays in the desired levels.

\subsection{Real-World Data}
\label{ssec:real_world}

Testing the controllers on real-world systems is difficult because of the health risks associated with the patients. We take the Diabetes dataset from UCI repository \cite{Dua:2019} which records the BG level and insulin dosage for $70$ patients. While testing controller is not possible here, hence we present the analysis regarding the modeling part. The long-range memory in the signals exist if the associated fractional exponent lies in the range of $(0, 0.5]$ as noted in Section\,\ref{ssec:modelEst}. In Fig.\,\ref{fig:uci_data_log_log}, we show the log-log plots of the variance of wavelets coefficients at various scales, for two subjects. We observe that the estimated value of $\alpha$ lies in $(0,0.5]$ which indicates presence of long-range memory, and hence fractional models can be used to make better predictions.

\subsection{Discussion}

Insulin dependent diabetes mellitus (IDDM) is a kind of chronic disease characterized by abnormal BG level. Topically, high level of BG, which is caused by either the pancreas does not compound enough insulin (a hormone that signals cells to uptake glucose in the bloodstream) or the produced insulin cannot be effectively used by the human body, can result in a disorder of metabolic that give rise to irreversible damage (such as lesion of patients' organs, retinopathy, nephropathy, peripheral neuropathy and blindness) \cite{tejedor2020reinforcement,derouich2002effect}. According to recent research, nowadays, IDDM is influencing 20-40 million people around the world and this amount is increasing over time \cite{you2016type}.  Related work \cite{diabetes1995resource} presents that tight blood glucose control along with insulin injections can help to control the disease, however intensive control can result in the risk of low blood sugar. This symptom can increase the risk of heart disease, or even sudden death.

\begin{figure}
	\centering
	\includegraphics[height = 2in]{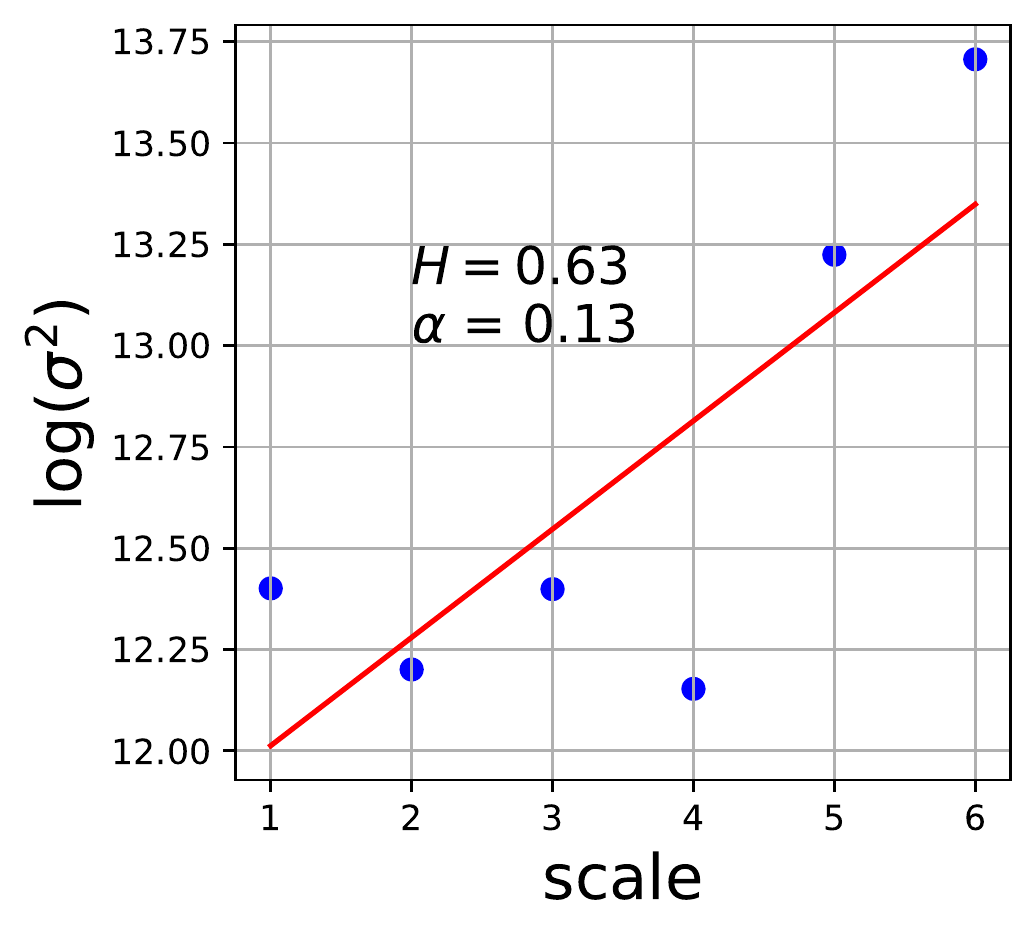}
	\includegraphics[height = 2in]{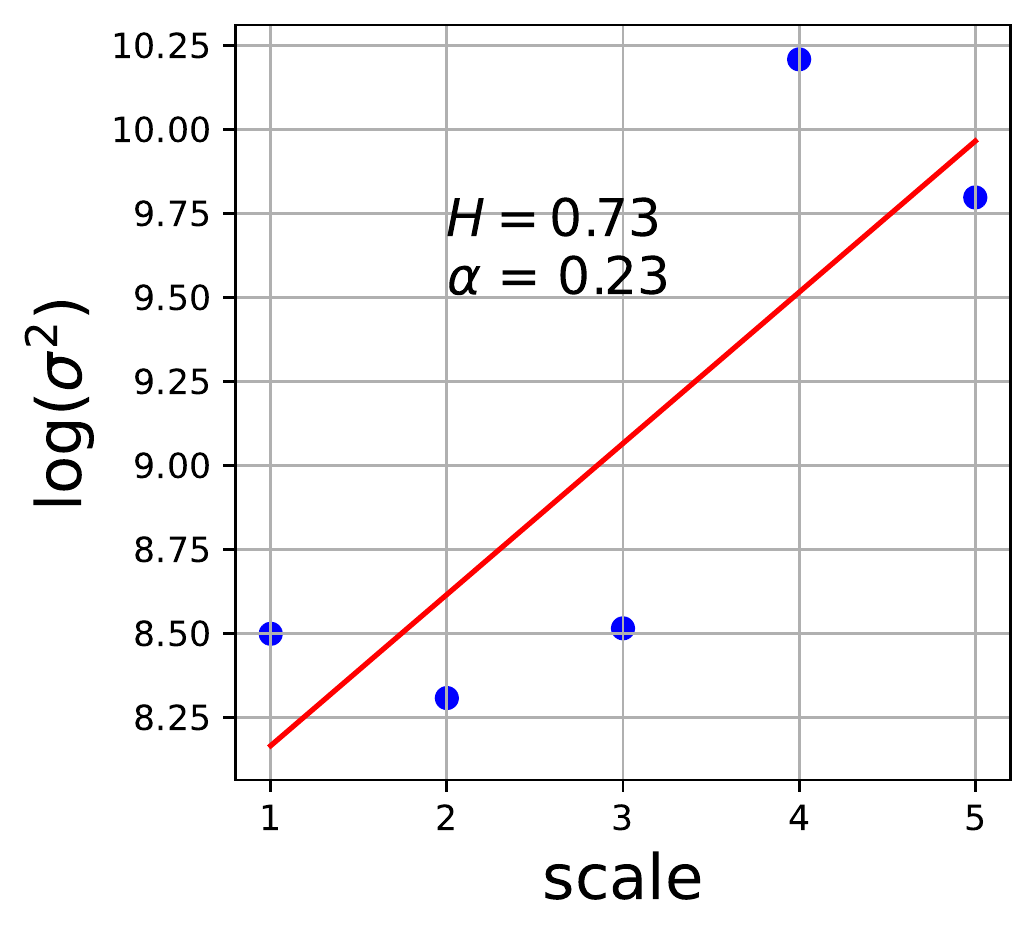}
	\caption{The log-log plot of variance of wavelet coefficients vs scale of two subjects in (a) and (b). The values of $\alpha$ lies in $(0,0.5]$ which indicates long-range correlations.}
	\label{fig:uci_data_log_log}
\end{figure}

To effectively and safely against with IDDM, in the work of \cite{coffen2009magnitude}, the authors constructed a CGM to detect the insulin among in the individuals in real time. This monitor can read the blood glucose of the patients for every 5 mins. Combining with an insulin pump (a small device that automatically inject insulin), the CGM can constructed a system called "artificial pancreas" (AP). The AP system is designed to control the symptom of patients which can dynamically predict the among of insulin the individual should be delivered. For many years, researchers have worked on designing efficient algorithms/models to correctly predict the required insulin of individuals in AP systems. In this paper, we present an innovative MBRL algorithm which is explored in the non-Markovian model to dynamically make the prediction of the among of insulin patients demand with high-accuracy and high-efficiency.  
\section{Conclusion}
\label{sec:concl}
There are many important learning control problems that are not
naturally formulated as Markov decision processes.  For example, if
the agent cannot directly observe the environment state, then the use
of a partially observable Markov decision process (POMDP)
\cite{smallwood1973} model is more appropriate.  Even in presence of
full observability, the probability distribution over next states may
not depend only on the current state.  A more general class can be
termed as History Dependent Process (HDP), which can be looked as
infinite-state POMDP \cite{leike2016nonparametric}. Another non-MDP
class for model-free is Q-value Uniform Decision Process (QDP)
\cite{ijcai2018-353}. The non-Markovianity in the rewards structure is
explored in \cite{gaon2019reinforcement, agarwal2021reinforcement}
which utilize model-free learning, and RL for POMDP is explored in
\cite{perez2017nonmarkovian} which is also model-free. 
MBRL is used for various robotics application \cite{Nagabandi_2018} in
the MDP setting. The deep probabilistic networks using MDP is used in
\cite{chua2018deep}. 

In this work, we constructed a non-Markovian Model Based Reinforcement Learning (MBRL) algorithm consisted with fractional dynamics model and the model predictive control. The current Reinforcement learning (RL) approaches have two kinds of limitations: (i) model-free RL models can achieve a high predict accuracy, but these approaches need a large number of data-points to train the model; (ii) current models don't make latent behavioral patterns into considerations which can affect the prediction accuracy in MBRL. We show that our non-Markovian MBRL model can validly avoid these limitations. Firstly, in our algorithm, we gather additional on-policy data to alternate between gathering the initial data, hence it needs less sample points than the general model-free RL approaches. Secondly, fractional dynamical model is the key element in our algorithm to improve/guarantee the prediction accuracy. The experiments on the blood glucose (BG) control to dynamically predict the desired insulin amount show that the proposed non-Markovian framework helps in achieving desired levels of BG for longer times with consistency. 

The richness of complex systems cannot be always modeled as Markovian dynamics. Previous works have shown that the long-range memory property of fractional differentiation operators can model biological signals efficaciously and accurately. Thus, we have modeled the blood glucose as non-Markovian fractional dynamical system and developed solutions using reinforcement learning approach. Finally, while the application of non-Markovian MBRL open venues for real-world implementation but proper care has to be taken especially when we have to deal with the healthcare systems. The future investigations would involve more personalized modeling capabilities for such systems with utilization of the domain knowledge. Nonetheless, we show that the use of long-range dependence in the biological models is worth exploring and simple models yield benefits of compactness as well as better accuracy of the predictions.

\nocite{bhardwaj2020blending, ross2012agnostic, kakade2003}

\bibliographystyle{IEEEtran}
\bibliography{IEEEabrv,fractionalRL}

\clearpage
\appendix

\allowdisplaybreaks
\allowbreak
\section{Proof of Theorem 1}
\label{appn:thm_proof}

For the notational purpose of the proof, we define the non-Markovian environment as $M$, and approximation to the environment as $\hat{M}$. The non-Markovian MPC based policy is $\hat{\pi}$, and the optimal policy is $\pi^{*}$. The approximation quality of the HDP dynamics are $\vert\vert\hat{P}_{h^{\prime}}(s^{\prime}\vert s, a) - P_{h}(s^{\prime}\vert s, a)\vert\vert_{1}\leq \mathcal{O}(t^{q})$, $\forall h, h^{\prime}\in \mathcal{H}_{t}$ with $h(t)=h^{\prime}(t)=s$, and the approximation for cost $\vert\vert c(s,a)-\hat{c}(s,a)\vert\vert_{\infty}\leq\varepsilon$. We also assume that the range of cost function is $[c_{min}, c_{max}]$. The initial history information provided is taken as $h_{0}\in\mathcal{H}_{0}$, for example, initial state $s_{0}$. We define $\hat{h} (h^{*})$ as the trajectories taken by policy $\hat{\pi}$ ($\pi^{*}$), respectively. We define the value function at step-$t$ ($0\leq t\leq T-1$), with history h using the model $M$, and policy $\pi$ as
\begin{equation}
    V_{t,M}^{\pi}(h) = \mathbb{E}_{h, \pi, M}\sum\limits_{k=t}^{T-1}\gamma^{k}(s_{k}, a_{k}),
\end{equation}
\noindent where $h\in\mathcal{H}_{t}$ and $h(t)=s_{t}$. We first present the simulation lemma for non-Markovian HDP as follows.
\begin{lemma}
Given an approximation $\hat{M}$ of the environment $M$ as, $\vert\vert\hat{P}_{h^{\prime}}(s^{\prime}\vert s, a) - P_{h}(s^{\prime}\vert s, a)\vert\vert_{1}\leq \mathcal{O}(t^{q})$, $\forall h, h^{\prime}\in \mathcal{H}_{t}$ with $h(t)=h^{\prime}(t)=s$, and the approximation for cost $\vert\vert c(s,a)-\hat{c}(s,a)\vert\vert_{\infty}\leq\varepsilon$, and any policy $\pi$ and history $h_{t}\in\mathcal{H}_{t}$,
\begin{align}
&\vert\vert \mathbb{E}_{h, \pi, \hat{M}}\sum\limits_{k=t}^{t+H-1}\gamma^{k}(s_{k}, a_{k}) -\mathbb{E}_{h, \pi, \hat{M}}\sum\limits_{k=t}^{t+H-1}\gamma^{k}(s_{k}, a_{k})\vert\vert_{\infty}\nonumber\\
&\qquad\leq\gamma^{t}H\left(\frac{c_{max}-c_{min}}{2}\right)\frac{1-\gamma^{H}}{1-\gamma}\mathcal{O}((t+H)^{q})  + \varepsilon\gamma^{t}\frac{1-\gamma^{H}}{1-\gamma}.
\end{align}
\end{lemma}
\begin{proof}
The difference in the value functions can be written as 
\begin{align*}
&\mathbb{E}_{h_{t}, \pi, \hat{M}}\sum\limits_{k=t}^{t+H-1}\gamma^{k}\hat{c}(s_{k}, a_{k}) - \mathbb{E}_{h_{t}, \pi, M}\sum\limits_{k=t}^{t+H-1}\gamma^{k}c(s_{k}, a_{k}) \\
& \qquad= \mathbb{E}_{h_{t}, \pi, \hat{M}}\sum\limits_{k=t}^{t+H-1}\gamma^{k}\hat{c}(s_{k}, a_{k}) - \mathbb{E}_{h_{t}, \pi, \hat{M}}\sum\limits_{k=t}^{t+H-1}\gamma^{k}c(s_{k}, a_{k}) \\
& \qquad+ \mathbb{E}_{h_{t}, \pi, \hat{M}}\sum\limits_{k=t}^{t+H-1}\gamma^{k}c(s_{k}, a_{k}) - \mathbb{E}_{h_{t}, \pi, M}\sum\limits_{k=t}^{t+H-1}\gamma^{k}c(s_{k}, a_{k}).
\end{align*}
With $\vert\vert c(s,a)-\hat{c}(s,a)\vert\vert_{\infty}\leq\varepsilon$, we can write that
\begin{align*}
&\vert\vert\mathbb{E}_{h_{t}, \pi, \hat{M}}\sum\limits_{k=t}^{t+H-1}\hat{c}(s_{k}, a_{k}) - \mathbb{E}_{h_{t}, \pi, \hat{M}}\sum\limits_{k=t}^{t+H-1}c(s_{k}, a_{k})\vert\vert_{\infty} \\
& \qquad\leq \mathbb{E}_{h_{t}, \pi, \hat{M}}\sum\nolimits_{k=t}^{t+H-1}\gamma^{k}\varepsilon  = \varepsilon\frac{1-\gamma^{H}}{1-\gamma}.
\end{align*}
For the remaining terms,  we can write that
\begin{align*}
&\mathbb{E}_{h_{t}, \pi, \hat{M}}\sum\limits_{k=t}^{t+H-1}\gamma^{k}c(s_{k}, a_{k}) - \mathbb{E}_{h_{t}, \pi, M}\sum\limits_{k=t}^{t+H-1}\gamma^{k}c(s_{k}, a_{k}) \\
& \quad= \left(\sum\limits_{s_{k}\sim h_{t},\pi,\hat{M}}{P}(s_{t}, \hdots, s_{t+H-1})-\sum\limits_{s_{k}\sim h_{t},\pi,M}{P}(s_{t}, \hdots, s_{t+H-1})\right)(\sum\limits_{k=t}^{t+H-1}\gamma^{k}c(s_{k}, a_{k})-\delta),
\end{align*}
\noindent where in the last equality, we can insert $\delta$ as 
\begin{align*}
&\sum\limits_{s_{k}\sim h_{t},\pi,\hat{M}}{P}(s_{t}, \hdots, s_{t+H-1}) = \sum\limits_{s_{k}\sim h_{t},\pi,M}{P}(s_{t}, \hdots, s_{t+H-1})=0.   
\end{align*}
Next,
\begin{align*}
&\vert\vert\mathbb{E}_{h_{t}, \pi, \hat{M}}\sum\limits_{k=t}^{t+H-1}\gamma^{k}c(s_{k}, a_{k}) - \mathbb{E}_{h_{t}, \pi, M}\sum\limits_{k=t}^{t+H-1}\gamma^{k}c(s_{k}, a_{k})\vert\vert_{\infty} \\
&\quad\leq\vert\vert(\sum\limits_{s_{k}\sim h_{t},\pi,\hat{M}}{P}(s_{t}, \hdots, s_{t+H-1})-\sum\limits_{s_{k}\sim h_{t},\pi,M}{P}(s_{t}, \hdots, s_{t+H-1}))\vert\vert_{\infty}\,\,\vert\vert\sum\limits_{k=t}^{t+H-1}\gamma^{k}c(s_{k}, a_{k})-\delta\vert\vert_{\infty}.
\end{align*}
By choosing $\delta = \sum\nolimits_{k=t}^{t+H-1}\gamma^{k}(\frac{c_{max}+c_{min}}{2})$, and upper-bounding difference of transition dynamics as $\mathcal{O}((t+H)^{q})$, we get the final expression.
\end{proof}

Now, we begin the proof of the Theorem\,1 as follows.
\begin{align*}
V_{t,M}^{\hat{\pi}}(\hat{h}_{t}) - V_{t,M}^{\pi^{\ast}}(h^{*}_{t}) &=
 \mathbb{E}_{\hat{h}_{t}, \hat{\pi}, M}\sum\limits_{k=t}^{t+H-1}\gamma^{k}c(s_{k}, a_{k}) - \mathbb{E}_{h^{*}_{t}, \pi^{*}, M}\sum\limits_{k=t}^{t+H-1}\gamma^{k}c(s_{k}, a_{k}) \\
& \qquad+ V_{t+H,M}^{\hat{\pi}}(\hat{h}_{t+H}) - V_{t+H,M}^{\pi^{\ast}}(h^{*}_{t+H}).
\end{align*}
The first set of terms can be expanded as follows.
\begin{align*}
&\mathbb{E}_{\hat{h}_{t}, \hat{\pi}, M}\sum\limits_{k=t}^{t+H-1}\gamma^{k}c(s_{k}, a_{k}) - \mathbb{E}_{h^{*}_{t}, \pi^{*}, M}\sum\limits_{k=t}^{t+H-1}\gamma^{k}c(s_{k}, a_{k})\\
&\qquad= \mathbb{E}_{\hat{h}_{t}, \hat{\pi}, M}\sum\limits_{k=t}^{t+H-1}\gamma^{k}c(s_{k}, a_{k}) - \mathbb{E}_{\hat{h}_{t}, \hat{\pi}, \hat{M}}\sum\limits_{k=t}^{t+H-1}\gamma^{k}\hat{c}(s_{k}, a_{k})\\
&\qquad\qquad + \mathbb{E}_{\hat{h}_{t}, \hat{\pi}, \hat{M}}\sum\limits_{k=t}^{t+H-1}\gamma^{k}\hat{c}(s_{k}, a_{k}) - \mathbb{E}_{\hat{h}_{t}, \pi^{*}, \hat{M}}\sum\limits_{k=t}^{t+H-1}\gamma^{k}\hat{c}(s_{k}, a_{k})\\
&\qquad\qquad + \mathbb{E}_{\hat{h}_{t}, {\pi}^{*}, \hat{M}}\sum\limits_{k=t}^{t+H-1}\gamma^{k}\hat{c}(s_{k}, a_{k}) - \mathbb{E}_{\hat{h}_{t}, \pi^{*}, \hat{M}}\sum\limits_{k=t}^{t+H-1}\gamma^{k}c(s_{k}, a_{k})\\
&\qquad\qquad + \mathbb{E}_{\hat{h}_{t}, {\pi}^{*}, \hat{M}}\sum\limits_{k=t}^{t+H-1}\gamma^{k}c(s_{k}, a_{k}) - \mathbb{E}_{{h}^{*}_{t}, \pi^{*}, M}\sum\limits_{k=t}^{t+H-1}\gamma^{k}c(s_{k}, a_{k}).
\end{align*}
Since the $\hat{\pi}$ is greedy policy that optimize \eqref{eqn:mpc_general}, therefore, 
\begin{equation}
 \mathbb{E}_{\hat{h}_{t}, \hat{\pi}, \hat{M}}\sum\limits_{k=t}^{t+H-1}\gamma^{k}\hat{c}(s_{k}, a_{k}) \leq \mathbb{E}_{\hat{h}_{t}, \pi^{*}, \hat{M}}\sum\limits_{k=t}^{t+H-1}\gamma^{k}\hat{c}(s_{k}, a_{k}).
\end{equation}
Now, using lemma\,1 and the approximation quality of costs, we get 
\begin{align*}
&\vert\vert V_{t,M}^{\hat{\pi}}(\hat{h}_{t}) - V_{t,M}^{\pi^{\ast}}(h^{*}_{t})\vert\vert_{\infty}\leq \\
&\quad 2\gamma^{t}H\left(\frac{c_{max}-c_{min}}{2}\right)\frac{1-\gamma^{H}}{1-\gamma}\mathcal{O}((t+H)^{q}) + 2\gamma^{t}\varepsilon\frac{1-\gamma^{H}}{1-\gamma}+\vert\vert V_{t+H,M}^{\hat{\pi}}(\hat{h}_{t+H}) - V_{t+H,M}^{\pi^{\ast}}(h^{*}_{t+H})\vert\vert_{\infty}.
\end{align*}
By adding the terms from $t=0$ till $T-1$, the terms that are $H$ apart cancels out in a telescopic sum fashion. Finally, we can write that
\begin{align*}
\vert\vert V_{0,M}^{\hat{\pi}}({h}_{0}) - V_{0,M}^{\pi^{\ast}}(h_{0})\vert\vert_{\infty} &\leq\sum\limits_{t=0}^{T-1}2\gamma^{t}H\left(\frac{c_{max}-c_{min}}{2}\right)\frac{1-\gamma^{H}}{1-\gamma}\mathcal{O}((t+H)^{q}) + 2\gamma^{t}\varepsilon\frac{1-\gamma^{H}}{1-\gamma}\\
&\leq 2\frac{1-\gamma^{H}}{1-\gamma}\left(\frac{c_{max}-c_{min}}{2}\right)H\mathcal{O}(T^{q})\nonumber\\
&\qquad + 2\varepsilon\frac{1-\gamma^{H}}{1-\gamma}\frac{1-\gamma^{T}}{1-\gamma}.
\end{align*}
\qed
\section{Fractional MPC}
\label{appn:fracMPCQprog}
The fractional MPC in Section\,\ref{ssec:frac_mpc} has linear constraints. Define the optimization variable $x = [\bar{s}[k+H]^T, \bar{s}[k+H-1]^{T},\hdots,\bar{s}[0]^{T}, a[k+H-1]^{T}, \hdots, a[k]^{T}]^{T}$. The first set of constraints can then be written as $\Theta\,x=b$, where
\setcounter{MaxMatrixCols}{20}
\begin{align*}
\Theta =
\begin{bmatrix}
I & D(\alpha, 1)+A & D(\alpha,2) & \hdots  &&&& D(\alpha,k+H) & B & 0 & \hdots & 0 \\
0 & I & D(\alpha, 1)+A  & D(\alpha,2) &&&& \hdots & 0 & B & \hdots &0\\
\vdots & \vdots & \ddots & \hdots & I & D(\alpha, 1)+A & \hdots & D(\alpha,k+1) & 0 & 0 & \hdots & B\\
\end{bmatrix},
\end{align*}
\noindent using,
\begin{equation}
D(\alpha, j)= diag(\psi(\alpha_{1},j), \hdots, \psi(\alpha_{n},j)).
\end{equation}
and $b = [e[k+H-1]^T, \hdots, e[k]^T, 0^T, \hdots, 0^T]^T$. The history equality constraint can be set as $\Phi\,x=d$ with $\Phi = [0,  I,  0]$ of appropriate size, and $d = [0^T, \hdots, s[k]^T, \hdots, s[0]^T, 0^T, \hdots, 0^T]^T$. Using these two equality constraints, and boundary limits for $\bar{s}[k]$, we get a quadratic programming with approximated costs $\hat{c}$ as quadratic function. For other convex versions of the approximated cost, a convex optimization with the above linear constraints can be formulated.

\end{document}